\documentclass[oribibl]{llncs}

\usepackage{amsmath,amssymb,amsfonts}
\usepackage{color}
\usepackage{algorithm}
\usepackage[noend]{algorithmic}
\usepackage{graphicx}
\usepackage{verbatim}
\usepackage{subfig}	
\usepackage{setspace}
\usepackage{epsfig,color}
\usepackage[nice]{nicefrac}
\usepackage{breqn}
\usepackage{hyperref}
\usepackage{xspace}
\usepackage{url}
\usepackage{lineno}
\usepackage{multirow}

\newcommand{\ignore}[1]{}

\def\vor{\text{Vor}}

\def\C{\mathcal{C}}

\def\O{\mathcal{O}}
\def\I{\mathcal{I}}

\def\I{\mathcal{I}}
\def\T{\mathcal{T}}

\def\orad{\mathcal{O}_{D}}

\def\dG{\mathbb{G}}
\def\dV{\mathbb{V}}

\def\dE{\mathbb{E}}
\def\dR{\mathbb{R}}

\newcommand{\mpl}{motion planning\xspace}

\newcommand{\cs}{C-space\xspace}

\newcommand{\Cpp}{C\raise.08ex\hbox{\tt ++}\xspace}

\newcommand{\lp}{local connector\xspace}

\newtheorem{observation}{Observation}

\newboolean{ICRA}
\newboolean{ARXIV}

\setboolean{ICRA}{false}
\ifthenelse{\boolean{ICRA}}
	{\setboolean{ARXIV}{false} }
	{\setboolean{ARXIV}{true}  }
\newcommand{\textVersion}[2]
{\ifthenelse{\boolean{ICRA} }{#1}{}\ifthenelse{\boolean{ARXIV}}{#2}{}}

\ifthenelse{\boolean{ICRA} }
{

}
{
}
\ifthenelse{\boolean{ARXIV} }
{

}
{
}

  \renewenvironment{thebibliography}[1]{%
    \begin{oldthebibliography}{#1}%
      \footnotesize
      \setlength{\parskip}{0ex}%
      \setlength{\itemsep}{0ex}%
  }%
  {%
    \end{oldthebibliography}%
  }
\def\mypart#1#2{%
  \par\break 
  \vskip .7\vsize 
  \refstepcounter{part}
  {\centering\Huge Part \thepart.\par \\
   \centering #1
  }%
  \vskip .1\vsize 

  #2
  \vfill\break 
}

\def\mypart#1#2
{
	\par\newpage\clearpage 
	\vspace*{5cm} 
	\refstepcounter{part}
	{
		\centering
		\textbf
		{
			\Huge Part \thepart
		}
		\par
	}
	\vspace{1cm}
	{
		\centering
		\textbf
		{
			\Huge #1
		}
		\par
	}
	\vspace{4cm}
	#2
	\vfill
	\pagebreak 
}

\newcommand{\argmin}{\operatornamewithlimits{argmin}}

\newboolean{ShowTODO}
\setboolean{ShowTODO}{true}

\ifthenelse{\boolean{ShowTODO}}{%

  \def\marrow{{\raggedright\footnotesize $\longleftarrow$}}

  \def\kiril#1{\textcolor{blue}{{\sc Kiril: }{\marrow\sf #1}}}

}{%

  \def\kiril#1{}

  \def\oren#1{}
}

\begin{document}

\pagestyle{plain}
\pagenumbering{arabic}

\title{Finding a needle in an exponential haystack: Discrete RRT for exploration of implicit roadmaps in multi-robot motion planning\thanks{This work has been supported in part by the 7th
						Framework Programme for Research of the European Commission, under
						FET-Open grant number 255827 (CGL---Computational Geometry
						Learning), by the Israel Science Foundation (grant no. 1102/11),
						by the German-Israeli Foundation (grant no. 1150-82.6/2011), and
						by the Hermann Minkowski--Minerva Center for Geometry at Tel Aviv
						University.}}

\titlerunning{Finding a Needle in an Exponential Haystack}

\newcommand*\samethanks[1][\value{footnote}]{\footnotemark[#1]}

\author{Kiril Solovey\thanks{K. Solovey and O. Salzman contributed equally to this paper.} \and Oren Salzman\samethanks \and Dan Halperin}

\institute{Blavatnik School of Computer Science,
			Tel-Aviv University, Israel}

\maketitle

\begin{abstract}
We present a sampling-based framework for multi-robot motion planning which combines an implicit representation of a roadmap with a novel approach for pathfinding in geometrically embedded graphs tailored for our setting.
Our pathfinding algorithm, \emph{discrete-RRT} (dRRT), is an adaptation of the celebrated RRT 
algorithm for the discrete case of a graph, and it enables a rapid exploration of the 
high-dimensional configuration space by carefully walking through an implicit representation of a 
tensor product of roadmaps for the individual robots.
We demonstrate our approach experimentally on scenarios
of up to 60 degrees of freedom
where our algorithm is faster by a factor of at least
ten
when compared to existing algorithms that we are aware of.
\end{abstract}

\section{Introduction}
    \emph{Multi-robot motion planning} is a fundamental problem in robotics and has been extensively studied. In this work we are concerned with finding paths for a group of robots, operating in the same workspace, moving from start to target positions while avoiding collisions with obstacles as well as with each other.
We consider the continuous formulation of the problem, where the robots and obstacles are geometric entities and the robots operate in a configuration space, e.g., $\dR^d$ (as opposed to the discrete variant, sometimes called the \textVersion
{
\emph{pebble motion} problem~\cite{LB11},
}
{
\emph{pebble motion} problem~\cite{ampp-ltafpm, gh-mcpm,k-cpmg, LB11},
}
where the robots move on a graph).
Moreover, we assume that each robot has its own start and target positions, as opposed to the unlabeled case (see, e.g.,~\cite{abhs-emrmp13,kh-pim05,sh-kcmr,TMK13}).

    \subsection{Previous work}
        We assume familiarity with the basic terminology of motion planning. For background, see, e.g., \cite{clhbkt-prmp,l-pa}. Initial work on motion planning aimed to develop \emph{complete} algorithms, which guarantee to find a solution when one exists or report that none exists otherwise. Such algorithms for the multi-robot case exist~\cite{ss-pm3,ss-cmp91,y-cms84} yet are exponential in the number of robots. The exponential running time,
\textVersion
{which may be unavoidable~\cite{hss-cmpmio, SY84}}
{which may be unavoidable~\cite{hss-cmpmio, SY84}}
can be attributed to the high number of \emph{degrees of freedom} (\emph{dof})---the sum of the dofs of the individual robots.
\textVersion
{
Several techniques to reduce the number of dofs exist~\cite{avbsv-mpfmr, bslm-cppmr}.
Alternatively, decoupled planners (see, e.g., \cite{lls-mpcmr, bo-pmpmr}) provide an efficient approach for a restricted set of problems.
}
{

        For two or three robots, the number of dofs may be slightly reduced~\cite{avbsv-mpfmr}, by constructing a path where the robots move while maintaining contact with each other.
A more general approach to reduce the number of dofs was suggested by van den Berg et al.~\cite{bslm-cppmr}.
In their work, the motion-planning problem is decomposed into subproblems, each consisting of a subset of robots, where every subproblem can be solved separately and the results can be combined into a solution for the original problem.

        Decoupled planners are an alternative to complete planners trading completeness for efficiency. Typically, decoupled planners solve separate problems for individual robots and combine the individual solutions into a global solution (see, e.g., \cite{bo-pmpmr,lls-mpcmr}). Although efficient in some cases, the approach usually works only for a restricted set of problems.
}

        The introduction of \emph{sampling-based} algorithms such as the \emph{probabilistic  roadmap \hyphenation{road-map} method} (PRM)~\cite{kslo-prm}, the \emph{rapidly-exploring random trees}~(RRT)~\cite{kl-rrtc}
and their many variants, had a significant impact on the field of \mpl due to their efficiency, simplicity and applicability to a wide range of problems.
Sampling-based algorithms attempt to capture the connectivity of the
\emph{configuration space} (\cs) by sampling collision-free configurations and constructing a \emph{roadmap}---a graph data structure where the free configurations are vertices and the edges represent collision-free paths between nearby configurations.
Although these algorithms are not complete, most of them are \emph{probabilistically complete}, that is, they are guaranteed to find a solution, if one exists, given a sufficient amount of time.
        Recently, Karaman and Frazzoli~\cite{KF11} introduced several variants of these algorithms such that, with high probability they produce paths that are \emph{asymptotically optimal} with respect to some quality measure.

        Sampling-based algorithms can be easily extended to the multi-robot case by considering the fleet of robots as one composite robot~\cite{sl-upp}. Such a naive approach suffers from inefficiency as it overlooks aspects that are unique to the multi-robot problem.
        More tailor-made sampling-based techniques have been proposed for the multi-robot case~\cite{hh-hmp, shh-mms2, sh-kcmr}. Particularly relevant to our efforts is the work of \v{S}vestka and Overmars~\cite{so-cppmr} who suggested to construct a composite roadmap which is a Cartesian product of roadmaps of the individual robots. Due to the exponential nature of the resulting roadmap, this technique is only applicable to problems that involve a modest number of robots.
        A recent work by Wagner et al.~\cite{wkh-ppp} suggests that the composite roadmap does not necessarily have to be explicitly represented. Instead, they maintain an implicitly represented composite roadmap, and apply their M* algorithm~\cite{wc-mstar} to efficiently retrieve paths, while minimizing the explored portion of the roadmap.
The resulting technique is able to cope with a large number of robots, for certain types of scenarios. Additional information on these two approaches is provided in Section~\ref{sec:c_roadmaps} below.

    \subsection{Contribution}
        We present a sampling-based algorithm for the multi-robot motion-planning problem called \emph{multi-robot discrete RRT} (MRdRRT). Similar to the approach of Wagner et al.~\cite{wkh-ppp}, we maintain an implicit representation of the composite roadmap. We propose an alternative, highly efficient, technique for pathfinding in the roadmap, which can cope with scenarios that involve tight coupling of the robots. Our new approach, which we call dRRT, is an adaptation of the celebrated RRT algorithm~\cite{kl-rrtc} for the discrete case of a graph, embedded in Euclidean space\footnote{We mention that we are not the first to consider RRTs in discrete domains. Branicky et al.~\cite{BCLM03} applied the RRT algorithm to a discrete graph. However, a key difference between the approaches is that we assume that the graph is \emph{geometrically embedded}, hence we use \emph{random points} as samples while they use nodes of the graph as samples. Additionally, their technique requires that all the neighbors of a visited vertex will be considered---a costly operation in our setting, as mentioned above.}.
dRRT traverses a composite roadmap that may have exponentially many neighbors (exponential in the number of robots that need to be coordinated).
The efficient traversal is achieved by retrieving only partial information of the explored roadmap.
Specifically, it considers a single neighbor of a visited vertex at each step. dRRT rapidly explores the \cs represented by the implicit graph. Integrating the implicit representation of the roadmap allows us to solve multi-robot problems while exploring only a small portion of the \cs.

        We demonstrate the capabilities of MRdRRT on the setting of polyhedral robots translating and rotating in space amidst polyhedral obstacles.
We provide experimental results on several challenging scenarios, where MRdRRT is faster by a factor of at least ten when compared to existing algorithms that we are aware of.
We show that we manage to solve problems of up to 60 dofs for highly coupled scenarios (Figure~\ref{fig:3d_scenarios}).

        The organization of this paper is as follows.
In Section~\ref{sec:c_roadmaps} we elaborate on two sampling-based multi-robot motion planning algorithms, namely the composite roadmap approach by \v{S}vestka and Overmars~\cite{so-cppmr} and the work on subdimensional expansion and M* by Wagner et al.~\cite{wc-mstar, wkh-ppp}.
In Section~\ref{sec:d_rrt} we introduce the dRRT algorithm. For clarity of exposition, we first describe it as a general pathfinding algorithm for geometrically embedded graphs. In the following section (Section~\ref{sec:implementation}) we describe the MRdRRT method where dRRT is used in the setting of multi-robot motion-planning problem for the exploration of the implicitly represented composite roadmaps.
We show in Section~\ref{sec:experimental_results} experimental results for the algorithm on different scenarios and conclude the paper in
Section~\ref{sec:discussion} with possible future research directions.

\begin{figure}[th!]
  \centering
  \subfloat
   [\sf Twisty]
   {
    \includegraphics[width=0.45\textwidth]{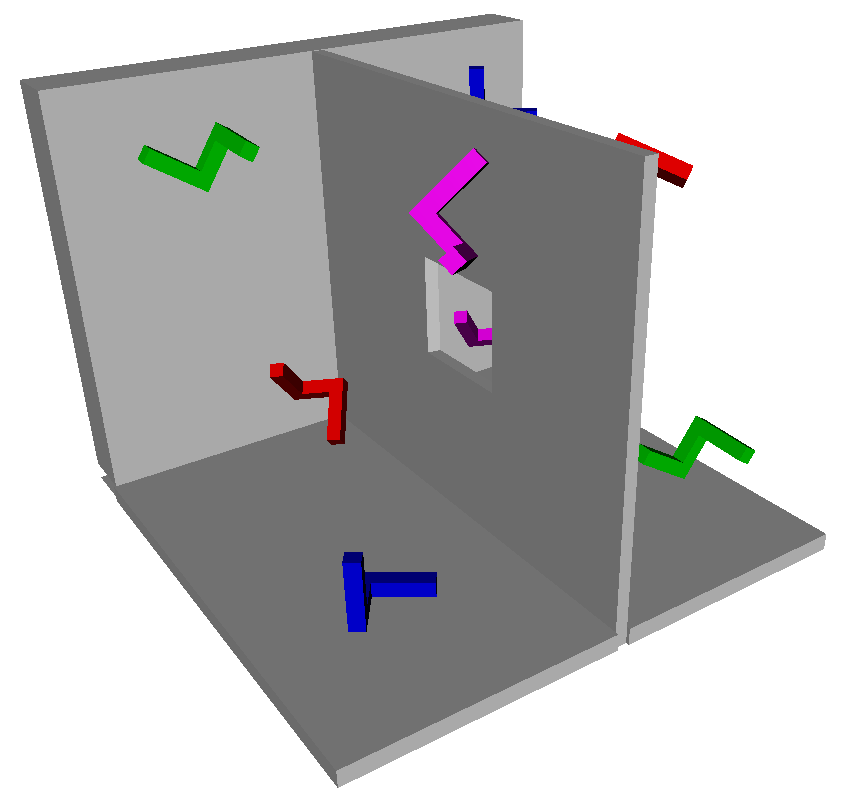}
   }
   \hspace{3mm}
   \subfloat
   [\sf Abstract ]
   {
    \includegraphics[width=0.45\textwidth]{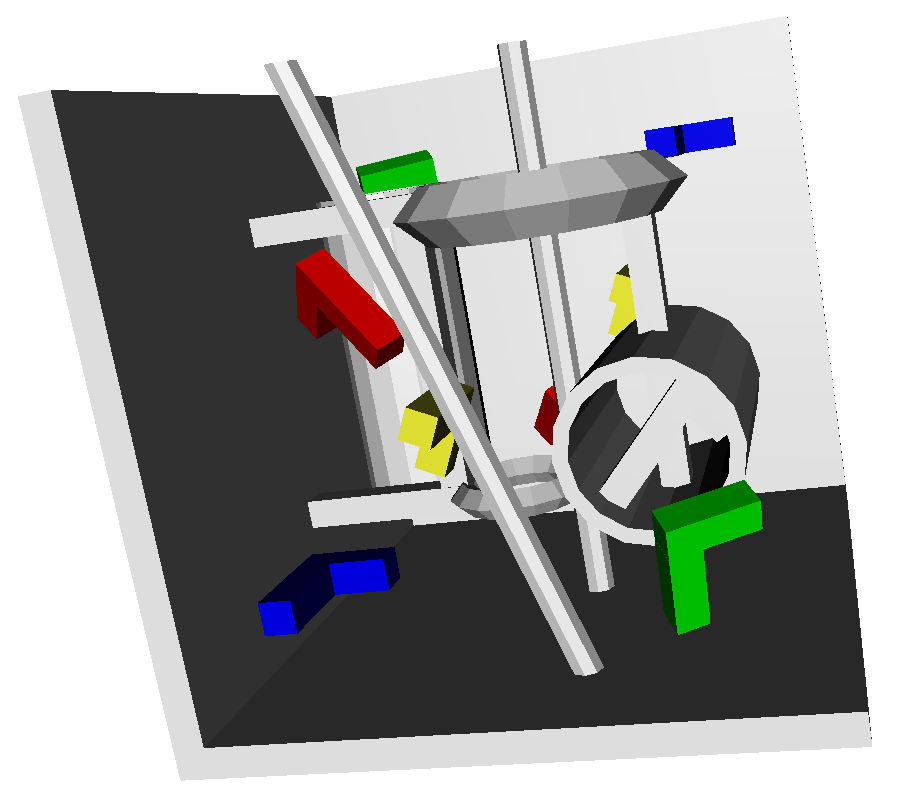}
   }
   \hspace{3mm}
  \subfloat
   [\sf Cubicles ]
   {
    \includegraphics[width=0.45\textwidth]{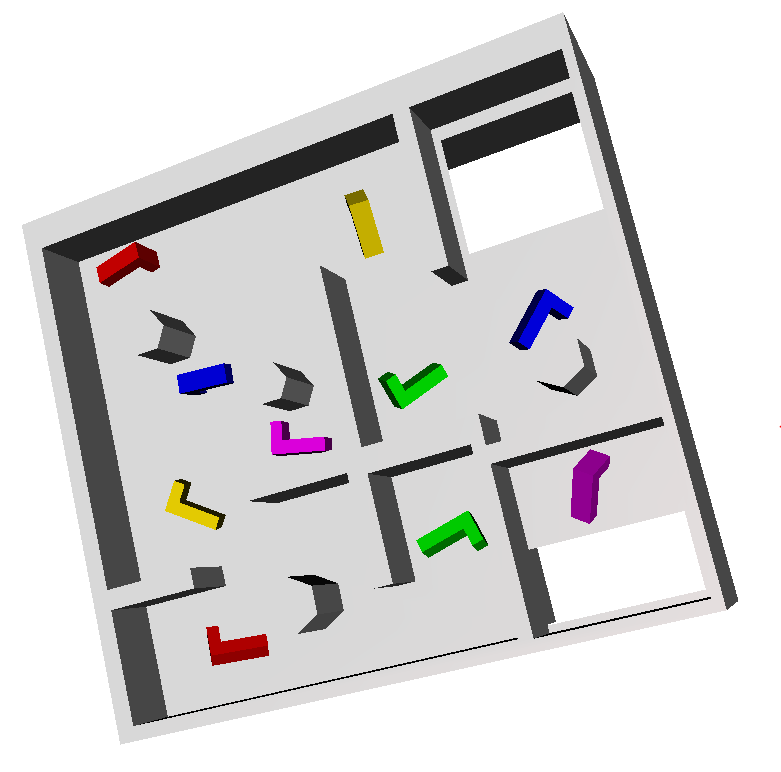}
   }
   \hspace{3mm}
  \subfloat
   [\sf Home ]
   {
   	\includegraphics[width=0.45\textwidth]{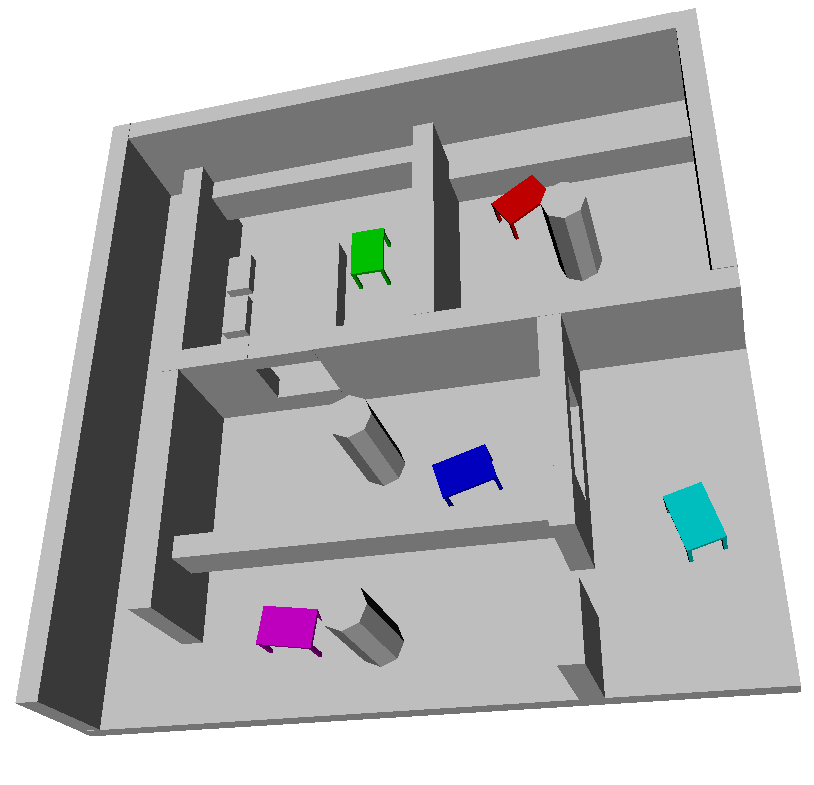}
   }
  \caption{\sf 	3D environments with robots that are allowed to rotate and translate (6DOFs). In scenarios (a),(b),(c) robots of the same color need to exchange positions. (a) Twisty scenario with 8 corkscrew-shaped robots, in a room with a barrier. 
(b) Abstract scenario with 8 L-shaped robots.
(c) Cubicles scenario with 10 L-shaped robots.
(d) Home scenario with 5 table-shaped robots that are placed in different rooms. The goal is to change rooms in a clockwise order.
The scenario were constructed using meshes that are provided by the
  							Open Motion Planning Library~\cite{ompl} (OMPL 0.10.2)
  							distribution.}
  \vspace {-5mm}
  \label{fig:3d_scenarios}
\end{figure}

\section{Composite roadmaps for multi-robot motion planning}
\label{sec:c_roadmaps}
We describe the composite roadmap approach 
introduced by \v{S}vestka and Overmars~\cite{so-cppmr}.
Here, a Cartesian product of PRM roadmaps of individual robots is considered as a means of devising a roadmap for the entire fleet of robots. 
However, since they consider an explicit construction of this roadmap, their technique is applicable to scenarios that involve only a small number of robots.
To overcome this, Wagner et al. suggest~\cite{wc-mstar, wkh-ppp} to represent the roadmap \emph{implicitly} and describe a novel algorithm to find paths on this implicit graph.

Let $r_1,\ldots,r_m$ be $m$ robots operating in a workspace~$W$
with start and target configurations $s_i,t_i$. 
We wish to find paths for every robot from start to target, while avoiding  collision with obstacles as well as with the other robots.
Let $G_i=(V_i,E_i)$ be a PRM roadmap for $r_i$, $|V_i|=n$, and let $k$ denote the maximal degree of a vertex in any $G_i$.
In addition, assume that $s_i,t_i \in V_i$, and that $s_i,t_i$ reside in the same connected component of $G_i$. Given such a collection of roadmaps $G_1,\ldots, G_m$ a composite roadmap can be defined in two different ways---one is the result of a \emph{Cartesian product} of the individual roadmaps while in the other a \emph{tensor product} is used~\cite{wiki-gp}.

The \emph{composite roadmap} $\dG=(\dV, \dE)$ is defined as follows. The vertices $\dV$ represent all combinations of 
collision-free placements of the $m$ robots. Formally, a set of $m$ robot configurations $C=(v_1,\ldots,v_m)$ is a vertex of $\dG$ if for every $i$, $v_i\in V_i$, and in addition, when every robot $r_i$ is placed in $v_i$ the robots are pairwise collision-free. The Cartesian and tensor products differ in the type of edges in the resulting roadmap.
If the Cartesian product is used, then $(C,C')\in \dE$, where $C=(v_1,\ldots,v_m), C'\in(v'_1,\ldots,v'_m)$, if there exists $i$ such that $(v_i,v'_i)\in E_i$, for every $j\neq i$ it holds that $v_j=v'_j$, 
and $r_i$ does not collide with the other robots stationed at $v_j=v'_j$ while moving from $v_i$ to $v'_i$.
A tensor product generates many more edges. Specifically, $(C,C')\in\dE$ if $(v_i,v'_i)\in E_i$ for every~$i$, and the robots remain collision-free while moving on the respective single-graph edges.%
\textVersion
{
}
{
\footnote{There is wide consensus on the term {\it tensor product}
 as defined here, and less so on the term {\it Cartesian product}. As the latter has already been used before in the context of motion planning, we will keep using it here as well.}}
\vspace{5pt}

\noindent\textbf{Remark.}
\textVersion
{
Throughout this work we only consider the \emph{tensor} product composite roadmap.}
{
Throughout this work, unless stated otherwise, we refer to the \emph{tensor} product composite roadmap.}\vspace{5pt}

\textVersion
{
Given a composite roadmap $\dG$, it is left to find such a path between $S=(s_1,\ldots,s_m)$ and $T=(t_1,\ldots,t_m)$. 
Unfortunately, the number of vertices of $\dG$ alone may reach $O(n^m)$. 
One may consider the A* algorithm~\cite{p-hiss}
since it may not need to traverse all the vertices of graph. 
Yet, even A* needs to consider all the neighbors of a visited vertex, which in our case is $O(k^m)$ and may turn out to be prohibitively~large.
}
{
Note that by the  definition of $G_i$ and $\dG$ it holds that $S,T\in\dV$, where $S=(s_1,\ldots,s_m),T=(t_1,\ldots,t_m)$. The following observation immediately follows (for both product types).

\begin{observation}
Let $C_1,\ldots,C_{h}$ be a sequence of $h$ vertices of $\dG$  such that
$S=C_1, \ T=C_{h}$ and
for every two consecutive vertices $(C_{i},C_{i+1})\in \dE$.
Then, there exists a path for the robots from $S$ to $T$.
\end{observation}

Thus, given a composite roadmap $\dG$, it is left to find such a path between $S$ and $T$. Unfortunately, standard pathfinding techniques, which require the full representation of the graph, cannot be used since the number of vertices of $\dG$ alone may reach $O(n^m)$. 
One may consider the A* algorithm~\cite{p-hiss}%
\textVersion{}{, or its variants, as appropriate for the task,} since it may not need to traverse all the vertices of graph. 
A central property of A* is that it needs to consider all the neighbors of a visited vertex in order to guarantee that it will find a path eventually. Alas, in our setting, this turns out to be a significant drawback, since the number of neighbors of every vertex is $O(k^m)$.}

Wagner et al. propose an adaptation of A* to the case of a composite roadmap called M*~\cite{wc-mstar}.
Their approach exploits the observation that only the motion of some robots has to be coupled in typical scenarios.
Thus, planning in the joint \cs is only required for robots that have to be coupled, while the motion of the rest of the robots can be planned individually. Hence, their method dynamically explores low-dimensional search spaces embedded in the full \cs, instead of the joint high-dimensional \cs.
This technique is highly effective for scenarios with a low degree of coupling, and can cope with large fleets of robots in such settings. However, when the degree of coupling increases, we observed sharp increase in the running time of this algorithm, as it has to consider many neighbors of a visited vertex. 


\section{Discrete RRT} \label{sec:d_rrt}
    We describe a technique which we call \emph{discrete RRT} (dRRT) for pathfinding in implicit graphs that are embedded in a Euclidean space. 
    For clarity of exposition, we first describe dRRT without the technicalities related to motion planning. We add these details in the subsequent section.
    As the name suggests, dRRT is an adaptation of the RRT algorithm~\cite{kl-rrtc} for the purpose of exploring discrete geometrically-embedded graphs, instead of a continuous space.

    Since the graph serves as an approximation of some relevant portion of the Euclidean space, traversal of the graph can be viewed as a process of exploring the subspace.
    The dRRT algorithm rapidly explores the graph by biasing the search towards vertices embedded in unexplored regions of the space.

    Let $G=(V,E)$ be a graph where every $v\in V$ is embedded in a point in Euclidean space $\dR^d$ and every edge $(v,v')\in E$ is a line segment connecting the points.
    Given two vertices $s,t\in V$, dRRT searches for a path in~$G$ from~$s$ to~$t$.
    For simplicity, assume that the graph is embedded in $[0,1]^d$.

    Similarly to its continuous counterpart, dRRT grows a tree rooted in $s$ and attempts to connect it to $t$ to form a path from $s$ to $t$. As in RRT, the growth of the tree is achieved by extending it towards random samples in $[0,1]^d$. In our case though, vertices and edges that are added to the trees are taken from $G$, and we do not generate new vertices and edges along the way.

As $G$ is represented implicitly, the algorithm uses an oracle to retrieve information regarding neighbors of visited vertices. We first describe this oracle and then proceed with a full description of the dRRT algorithm. Finally, we show that this technique is \emph{probabilistically complete}.

    \subsection{Oracle to query the implicit graph}
In order to retrieve partial information regarding the neighbors of visited vertices, dRRT consults an oracle described below. We start with several basic definitions. 

Given two points $v,v'\in [0,1]^d$, denote by $\rho(v,v')$ the ray that starts in $v$ and goes through $v'$. Given three points $v,v',v''\in [0,1]^d$, denote by $\angle_v(v',v'')$ the (smaller) angle between $\rho(v,v')$ and $\rho(v,v'')$.

        \begin{definition}[Direction Oracle]
            Given a vertex $v\in V$, and a point $u\in [0,1]^d$ we define
                $$\orad(v,u):=\argmin_{v'}\left\{\angle_v(u,v')|(v,v')\in E\right\}.$$
        \end{definition}
        In other words, the direction oracle returns the neighbor $v'$ of $v$ such that the direction from $v$ to $v'$ is closest to the direction from $v$ to $u$.
    \subsection{Description of dRRT}
        At a high level, dRRT proceeds similar to the RRT algorithm, and we repeat it here for completeness. The dRRT algorithm (Algorithm~\ref{alg:planner}) grows a trees $\T$ which is a subgraphs of $G$ and is rooted in $s$ (line 1). The growth of $\T$ (line~3) is achieved by an expansion towards random samples. Additionally, an attempt to connect $\T$ with $t$ is made (line 4). The algorithm terminates when this operation succeeds and a solution path is generated (line~6), otherwise the algorithm repeats line 2.

        Expansion of $\T$ is performed by the EXPAND operation (Algorithm~\ref{alg:expand}) which performs $N$ iterations that consist of the following steps: A point $q_{\text{rand}}$ is sampled uniformly from $[0,1]^d$ (line~2). Then, a node $q_{\text{near}}$ that is the closest to the sample (in Euclidean distance), is selected (line~3). $q_{\text{near}}$ is extended towards the sample by locating the vertex $q_{\text{new}}\in V$, that is the neighbor of $q_{\text{near}}$ in $G$ in the direction of $q_{\text{rand}}$ (by the direction oracle~$\orad$). Once $q_{\text{new}}$ is found (line~4), it is added to the tree (line~6) with the edge $(q_{\text{near}},q_{\text{new}})$ (line 7). See an illustration of this process in Figure~\ref{fig:dRRT_alg}. This is already different from the standard RRT as we cannot necessarily proceed exactly in the direction of the random point.

        After the expansion, dRRT attempts to connect the tree $\T$ with $t$ using the CONNECT\_TO\_TARGET operation (Algorithm~\ref{alg:connect}). For every vertex $q$ of $\T$, which one of the $K$ nearest neighbors of $t$ in $\T$ (line~1), an attempt is made to connect $q$ to $t$ using the method LOCAL\_CONNECTOR (line~2) which is a crucial part of the dRRT algorithm
(see Subsection~\ref{sub:local}). 

        Finally, given a path from some node $q$ of $\T$ to $t$ the method RETRIEVE\_PATH (Algorithm~\ref{alg:planner}, line~6) returns the concatenation of the path from $s$ to $q$, with~$\Pi$.
        
        \begin{algorithm}[b!]
            \caption{dRRT\_PLANNER ($s,t$)}
            \label{alg:planner}
            \begin{algorithmic}[1]
                \STATE  $\T$.init($s$)
                \LOOP
                \STATE  EXPAND($\T$)
                \STATE  $\Pi \leftarrow$ CONNECT\_TO\_TARGET($\T,t$)
                \IF     {not\_empty($\Pi$)}
                    \RETURN RETRIEVE\_PATH($\T,\Pi$)
                \ENDIF
                \ENDLOOP
            \end{algorithmic}
        \end{algorithm}
        \begin{algorithm}[bh!]
            \caption{EXPAND ($\T$)}
            \label{alg:expand}
            \begin{algorithmic}[1]
                \FOR    {$i = 1 \to N$}
                    \STATE  $q_{\text{rand}}\leftarrow$ RANDOM\_SAMPLE()
                    \STATE  $q_{\text{near}}\leftarrow$
                        NEAREST\_NEIGHBOR($\T,q_{\text{rand}}$)
                    \STATE  $q_{\text{new}}\leftarrow \orad(q_{\text{near}},q_{\text{rand}})$
                    \IF     {$q_{\text{new}} \not\in \T$}
                            \STATE  $\T$.add\_vertex($q_{\text{new}}$)
                            \STATE  $\T$.add\_edge($q_{\text{near}} ,q_{\text{new}}$)
                    \ENDIF
                \ENDFOR
            \end{algorithmic}
        \end{algorithm}
        \begin{algorithm}[bh!]
            \caption{CONNECT\_TO\_TARGET($\T,t,$)}
            \label{alg:connect}
            \begin{algorithmic}[1]
                \FOR    {$q\in$ NEAREST\_NEIGHBORS($\T,t,K$)}
                    \STATE  $\Pi \leftarrow$ LOCAL\_CONNECTOR($q,t$)
                    \IF     {not\_empty($\Pi$)}
                            \RETURN     {$\Pi$}
                    \ENDIF
                \ENDFOR
                \RETURN     $\emptyset$
            \end{algorithmic}
        \end{algorithm}

    \subsection{Local connector}\label{sub:local}
We show in the following subsection that it is possible that $\T$ will eventually reach $t$ during the EXPAND stage, and therefore an application of LOCAL\_CONNECTOR will not be necessary.
However, in practice this is unlikely to occur within a short time frame, especially when $G$ is large. 
Thus, we employ a heavy-duty technique, which given two vertices $q_0,q_1$ of $G$ tries to find a path between them. We mention that it is common to assume in sampling-based algorithms that connecting nearby samples will require less effort than solving the initial problem and here we make a similar assumption.
We assume that a \lp is effective only on \emph{restricted} pathfinding problems, thus in the general case it cannot be applied directly on $s,t$, as it may be highly costly (unless the problem is easy). A concrete example of a \lp is provided in the next section.


    \subsection{Probabilistic completeness of dRRT}\label{sec:complete}
    Recall that an algorithm is \emph{probabilistically complete} if the probability it finds a solution tends to one as the run-time of the algorithm tends to infinity (when such a solution exists). For simplicity, we show that dRRT possesses a stronger property and with high probability will reveal all the vertices of the traversed graph, assuming this graph is connected.

    The proof relies on the assumption that the vertices of the traversed graph $G$ are in \emph{general position}, that is, every pair of distinct vertices are embedded in two distinct points in $\dR^d$, and for every triplet of distinct vertices the points in which they are embedded are non-collinear. 
This issue will be addressed in the following section, where we consider the application of dRRT on a specific type of graphs. The proof does not need to take into consideration the \lp.

    \begin{theorem}\label{thm:complete}
         Let $G=(V,E)$ be a connected graph embedded in $[0,1]^d$ where the vertices are in general position. Then ,with high probability, every vertex of $G$ will be revealed by the dRRT algorithm, given sufficient amount of time.
    \end{theorem}
    \begin{proof}
        Denote by $U$ the set of vertices of $\T$ after the completion of an iteration of the algorithm. Let $v^*\in V\setminus U$ be an unvisited vertex such that there exists $(v,v^*)\in E$, where $v\in U$. We wish to show that the probability that $\T$ will be expanded on the edge $(v,v^*)$, and thus $v^*$ will be added to $U$, is bounded away from zero. For simplicity we assume that there exists a single vertex $v\in U$ that has an edge to~$v^*$.

        Denote by $\vor(v)$ the \emph{Voronoi cell}~\cite{bkos-cg08} of the site $v$, in the Euclidean (standard) Voronoi diagram of point sites, where the sites are the vertices of $U$ (Figure~\ref{fig:dRRT_alg}(b)). In addition, denote by $\vor'(v,v^*)$ the Voronoi cell of $\rho(v,v^*)$, in a Voronoi diagram of the ray sites $\rho(v,v^*),\rho(v,u_1),\ldots,\rho(v,u_j)$, where $u_1,\ldots,u_j$ are the neighbors of $v$ in $\T$, not including $v^*$ (Figure~\ref{fig:dRRT_alg}(c)). 

        Notice that in order to extend $\T$ from $v$ to $v^*$ the random sample $q_{\text{rand}}$ in EXPAND (Algorithm~\ref{alg:expand}) has to fall inside $\vor(v)\cap\vor'(v,v^*)$. Thus, in order to guarantee that $v^*$ will be added to $\T$, with non-zero probability, we show that the shared region between these two cells has non-zero measure, namely $|\vor(v)\cap\vor'(v,v^*)|>0$, where $|\Gamma|$ denotes the volume of $\Gamma$.

        By the general position assumption we can deduce that $|\vor(v)|>0$ and $|\vor'(v,v^*)|>0$. In addition, the intersection between the two cells is clearly non-empty:
        There is a ball with radius $r>0$ whose center is $v$ and is completely contained in $\vor(v)$; similarly, there is a cone of solid angle $\alpha >0$ with apex at $v$ fully contained in $\vor'(v,v^*)$.
        Hence, it holds that $|\vor(v)\cap\vor'(v,v^*)|>0$, otherwise $v$ and $v^*$ are embedded in the same point. \qed
    \end{proof}
    
We note that a more careful analysis can yield an explicit bound on the convergence rate of dRRT. Such a bound may be computed using the size of the smallest cell in the Voronoi diagram of all nodes of~$G$.

\begin{figure*}[bh!]
  \centering
  \subfloat
   [\sf ]
   {
    \includegraphics[width=0.3\textwidth, clip=true, trim=45pt 105pt 115pt 65pt]{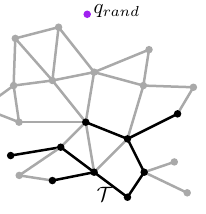}
   }
   \hspace{0.5mm}
  \subfloat
   [\sf ]
   {
    \includegraphics[width=0.3\textwidth, clip=true, trim=45pt 105pt 115pt 65pt]{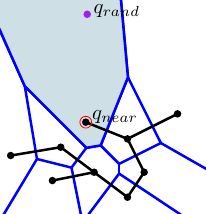}
   }
   \\
   \label{fig:drrt_2}
   \hspace{0.5mm}
  \subfloat
   [\sf ]
   {
   \includegraphics[width=0.3\textwidth, clip=true, trim=45pt 105pt 115pt 65pt]{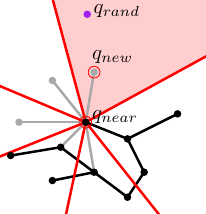}
   }
   \label{fig:drrt_3}
   \hspace{0.5mm}
  \subfloat
   [\sf ]
   {
    \includegraphics[width=0.3\textwidth, clip=true, trim=45pt 105pt 115pt 65pt]{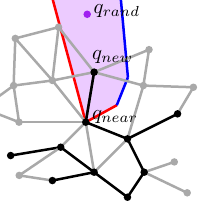}
   }
  \caption{\sf 	An illustration of the expansion step of dRRT. The tree $\T$ is drawn with black vertices and edges, while the gray elements represent the unexplored portion of the graph $G$.
  (a)~A random point $q_{\text{rand}}$ (purple) is drawn uniformly from $[0,1]^d$.
  (b)~The vertex $q_{\text{near}}$ of $\T$ that is the Euclidean nearest neighbor of $q_{\text{rand}}$ is extracted.
  (c)~The neighbor $q_{\text{new}}$ of $q_{\text{near}}$, such that its direction from $q_{\text{near}}$ is the closest to the direction of $q_{\text{rand}}$ from $q_{\text{near}}$, is identified.
  (d)~The new vertex and edge are added to $\T$. \emph{Additional information for Theorem~2}: In~(b) the Voronoi diagram of the vertices of $\T$ is depicted in blue, and the Voronoi cell of $q_{\text{near}}$ ,$\vor(q_{\text{near}})$, is filled with light blue. In~(c) the Voronoi diagram of the rays that leave $q_{\text{near}}$ and pass through its neighbors is depicted in red, and the Voronoi cell of $\rho(q_{\text{near}},q_{\text{new}})$, $\vor'(q_{\text{near}},q_{\text{new}})$, is filled with pink. The purple region in~(d) represents $\vor(q_{\text{near}})\cap\vor'(q_{\text{near}},q_{\text{new}})$.}
  \vspace {-5mm}
  \label{fig:dRRT_alg}
\end{figure*}

\section{Multi-robot motion planning with dRRT}
\label{sec:implementation}
In this section we describe the MRdRRT algorithm. Specifically, we discuss the adaptation of dRRT for pathfinding in a composite roadmap $\dG$, which is embedded in the joint \cs of $m$ robots.
   In particular, we show an implementation of the oracle $\orad$, which relies solely on the representation of $G_1,\ldots, G_m$.
   Additionally, we discuss an  implementation of the \lp component that takes advantage of the fact that $\dG$ represents a set of valid positions and movements of multiple robots.
   Finally, we discuss the probabilistic completeness of our entire approach to multi-robot motion planning.

    \subsection{Oracle $\orad$}
        Recall that given $C\in \dV$ and a random sample  $q$, $\O_D(C,q)$ returns $C'$ such that $C'$ is a neighbor of $C$ in $\dG$, and for every other neighbor $C''$ of $C$, $\rho(C,q)$ forms a smaller angle with $\rho(C,C')$ than with $\rho(C,C'')$, where $\rho$ is as defined in Section~\ref{sec:complete}.

        Denote by $\C(r_i)$ the \cs of $r_i$. Let $q=(q_1,\ldots,q_m)$ where $q_i\in\C(r_i)$, and let $C=(c_1,\ldots,c_m)$ where $c_i\in V_i$. To find a suitable neighbor for $C$ we first find the most suitable neighbor for every individual robot and combine the $m$ single-robot neighbors into a candidate neighbor for $C$. We denote by $c'_i=\orad(c_i,q_i)$ the neighbor of $c_i$ in $G_i$ that is in the direction of $q_i$. Notice that the implementation of the oracle for individual roadmaps is trivial---for example, by traversing all the neighbors of~$c_i$ in $G_i$.
        Let $C'=(c'_1,\ldots,c'_m)$ be a candidate for the result of $\O_D(C,q)$. If $(C,C')$ represents a valid edge in $\dG$, i.e., no robot-robot collision occurs, we return $C'$. Otherwise, $\O_D(C,q)$ returns $\emptyset$. In this case, the new sample is ignored and another sample is drawn in the EXPAND phase (Algorithm~\ref{alg:expand}).

The completeness proof of the dRRT (Theorem~\ref{thm:complete}) 
for this specific implementation of $\orad$, is straightforward. 
Notice that in order to extend $C=(c_1,\ldots,c_m)$ to $C'=(c'_1,\ldots,c'_m)$ the sample $q=(q_1,\ldots,q_m)$ must obey the following restriction: 
For every robot $r_i$, $q_i$ must lie in $\vor(c_i)\cap\vor'(c_i,c'_i)$ 
(where in the original proof we required that $q$ will lie in $\vor(C)\cap\vor'(C,C')$). 
Also note that the points in $\C(r_i)$ are in general position, as required by Theorem~\ref{thm:complete}, since they were uniformly sampled by PRM.

    \subsection{Local connector implementation}
        Recall that in the general dRRT algorithm the \lp is used for connecting two given vertices of a graph. As our \lp we rely on a framework described by van den Berg et al.~\cite{bslm-cppmr}. Given two vertices $\dV=(v_1,\ldots,v_m), \dV'=(v'_1,\ldots,v'_m)$ of $\dG$ we find for each robot $i$ a path $\pi_i$ on $G_i$ from $v_i$ to $v'_i$. The connector attempts to find an ordering of the robots such that robot $i$ does not leave its start position on $\pi_i$ until robots with higher priority reached their target positions on their respective path, and of course that it also avoids collisions. When these robots reach their destination robot $i$ moves along $\pi_i$ from $\pi_i(0)$ to $\pi_i(1)$. During the movement of this robot the other robots stay put.
        
        The priorities are assigned according to the following rule: if moving robot $i$ along $\pi_i$ causes a collision with robot $j$ that is placed in $v_j$ then robot $i$ should move \emph{after} robot $j$. Similarly, if $i$ collides with robot $j$ that is placed in $v'_j$ then robot $i$ should move \emph{before} robot $j$.  This prioritization induces a directed graph $\I$. In case this graph is acyclic we generate a solution according to the prioritization of the robots. Otherwise, we report failure.
        
        We decided to use this simple technique in our experiments due to its low cost, in terms of running time, regardless of whether it succeeds finding a solution or not. We wish to mention that we also experimented with M* with a bounded degree of coupling (to avoid considering exponentially many neighbors) as the local connector in our algorithm. However, the ordering algorithm 
of~\cite{bslm-cppmr} turned out to be considerably more efficient.
        
    \subsection{Probabilistic completeness of MRdRRT}
In order for the motion-planning framework to be probabilistically complete, we still need to show that
(i)~as the number of samples used for each single-robot roadmap tends to infinity, the composite roadmap will contain a path (if such a path exists) and 
(ii)~that the proof of Theorem~\ref{thm:complete} still holds when the size of the graph tends to infinity.
	Indeed, \v{S}vestka and Overmars~\cite{so-cppmr} show that the composite roadmap approach is probabilistically complete when the graph-search algorithm is complete.
However, in our setting, the graph-search algorithm is only probabilistically complete and the proof may need to be refined as the size of each Voronoi cell tends to zero.

We note that as the composite roadmap is finite, it is easy to modify the dRRT algorithm such that it will be complete.
This may be done by keeping a list of exposed nodes that still have unexposed edges. 
At the end of every iteration of the main loop of dRRT (Algorithm~\ref{alg:planner}, line 2) one node is picked from the list and one of its unexposed edges is exposed (finding an unexposed edge is done in a brute force manner). Although the above modification ensures completeness of dRRT and hence probabilistic completeness of MRdRRT, we are currently looking for an alternative proof that does not require 
altering the dRRT algorithm.

\section{Experimental results}
\label{sec:experimental_results}

We implemented MRdRRT for the case of polyhedral robots translating and rotating among polyhedral obstacles (see Figure~\ref{fig:3d_scenarios}). We compared the performance of MRdRRT with RRT and an improved (recursive) version of M* that appears in~\cite{wc-mstar}. To make the comparison as equitable as possible, as dRRT does not take into consideration the quality of the solution, we use the  \emph{inflated} version of M*~\cite{wc-mstar} with relaxed optimality guarantees.\vspace{5pt}

\noindent \textbf{Implementation details.} The algorithms were implemented in \Cpp.
The experiments were conducted on a laptop with an Intel i5-3230M 2.60GHz processor with 16GB of memory, running 64-bit Windows 7.
We implemented a generic framework for multi-robot motion planning based on composite roadmaps.
The implementation relies on 
PQP~\cite{pqp} for collision detection,
and performs nearest-neighbor queries
using the Fast Library for Approximate Nearest Neighbors (FLANN)~\cite{muja_flann_2009}.
Metrics, sampling and interpolation in the 3D environments followed the guidelines presented by Kuffner~\cite{Kuffner04}. To eliminate the dependence of dRRT on parameters we assigned them according to the number of iterations the algorithm performed so far, i.e., the number of times that the main loop has been repeated. Specifically, in the $i$'th iteration each EXPAND~(Algorithm~\ref{alg:expand}) call performs $2^i$ iterations ($N=2^i$), while CONNECT\_TO\_TARGET uses $K=i$ candidates that are connected with $t$.

\noindent \textbf{Test scenarios.} We report in Table~\ref{tbl:coupled_scenarios} the running times of M* and dRRT for the scenarios. The first three scenarios are especially challenging as they consist of a large number of robots, and require a substantial amount of coordination between them. 
The fourth scenario (``Home'') is more relaxed and consists of only five robots and requires little coordination. 

We ran each of the three algorithms 10 times on each scenario.
RRT proved incapable of solving any of the test scenarios, running for
several tens of minutes until terminating due to exceeding the memory limits.
We believe that RRT as-is is not suitable for high-dimensional, coupled, multi-robot motion planning.
M* exhibited slightly better performance.
For the first three scenraios, which involve multiple robots and require a substantial amount of coordination,
it never exceeded a success rate of 40\%.
In particular, it often ran out of memory or ran for a very long duration
(we terminated it if its running time exceeded ten times the running time of MRdRRT).
On the other hand, MRdRRT was stable in its results and managed to solve
all the scenarios for each of the 10 attempts.
When M* did manage to solve one of the first three scenarios,
it explored between 2.5 to 10 times the number of vertices that dRRT explored.
For the fourth scenario the results of MRdRRT and M*
were comparable and in general we found M* more suitable for situations where
only a small number of robots have to interact at any given time. We  mention that MRdRRT was unable to solve scenarios that consist of a substantially larger number of robot than we used in our experiments. We believe that it would be beneficial to consider a stronger \emph{local connector} in such cases.

\begin{table*}[t,b]
\begin{center}
	\begin{tabular}{c||c|c||c|c|c||c|c|c|c|c|}
\cline{2-11} & \multicolumn{2}{ c|| }{PRM} & \multicolumn{3}{c||}{M*} & \multicolumn{5}{ c| }{MRdRRT} \\ 
\cline{1-11} \multicolumn{1}{|c||}{\multirow{2}{*}{scenario}} & \multicolumn{1}{c|}{\multirow{2}{*}{$n$}} & \multicolumn{1}{c||}{\multirow{2}{*}{time}} & visited  & total  & success  & visited  & connect & expand & total & success \\
 \multicolumn{1}{ |c||  }{}  & \multicolumn{1}{ c | }{} & \multicolumn{1}{ c || }{} & vertices & time & rate & vertices &  time &  time & time & rate\\
 \cline{1-11} 
 \multicolumn{1}{|c||  }{Twisty}  & 8k & 10s & DNF  & DNF  & 0\% & 8k & 3.3s & 6.7s & 11s & 100\%\\
 \multicolumn{1}{|c||  }{Abstract}& 10k & 24.8s & 300k & 267s & 30\% & 34k & 30.4s & 25.5s & 55.9s & 100\%\\
 \multicolumn{1}{|c||  }{Cubicles} & 10k & 16.2s & 27k & 31s & 40\% & 12k & 16.3s & 36.8s & 53.1s & 100\%\\
 \multicolumn{1}{|c||  }{Home}& 5k & 10.1s & 2k & 3.9s  & 100\% & 8k & 1.5s & 2.9s & 4.4s & 100\%\\
 \cline{1-11} 
\end{tabular}
\end{center}

\caption{\sf  Results for M* and MRdRRT on the scenarios depicted
in Figure~\ref{fig:3d_scenarios}. We first report the number of vertices (reported in thousands) used in the construction of the single-robot PRM roadmaps and the elapsed time (all times reported are in seconds). Then we report the number of visited vertices, the total running time, and the success rate of M*. A similar report is given for dRRT, but we also specify the duration of the connection phase (using \lp) and the expansion phase. The running times and the amount of explored vertices are averaged over the number of successful attempts.}
\vspace {-9mm}
\label{tbl:coupled_scenarios}
\end{table*}

\section{Discussion}\label{sec:discussion}
In this section we state the benefits of MRdRRT, which consists of an implicitly represented roadmaps for multi-robot motion planning combined with an efficient approach for pathfinding for such roadmaps.

Recall that the implicitly-represented composite roadmap $\dG$ results from a tensor product of $m$ PRM roadmaps $G_1,\ldots,G_m$. The reliance on the precomputed individual roadmaps eliminates the need to perform additional collision checking between robots and obstacles while querying $\dG$. This has a substantial impact on the performance of MRdRRT as it is often the case that checking whether $m$ robots collide with obstacles is much more costly than checking whether the $m$ robots collide between themselves. This is in contrast with more naive approaches,
such as RRT which consider the group of robots as one
large robot.
In such cases, checking whether a configuration
(or an edge) is collision free requires
checking for the two types of collisions simultaneously.

The M* algorithm, which also uses the underlying structure of~$\dG$, performs very well in situations where only a small subset of the robots need to coordinate. In these situations it can cope, almost effortlessly, with several tens of robots while outperforming our framework. However, in scenarios where a substantial amount of coordination is required between the robots M* suffers from a disadvantage, since it is forced to consider exponentially many neighbors when performing the search on $\dG$. 
In contrast, dRRT performs a ``minimalistic'' search and advances in small steps, little by little, regardless of the difficulty of the problem at hand. Moreover, dRRT strives to reach unknown regions in $\dG$ while avoiding spending too much time in the exploration of regions that are in the vicinity of explored vertices. This is done via the Voronoi bias, as shown in the proof of Theorem~\ref{thm:complete}. This is extremely beneficial when working on $\dG$ since it contains vertices which represent essentially the same conformation of the robots, and thus considering many vertices within a small region would not lead to a better understanding of the problem at hand. To justify this claim, consider the following example. Suppose that for every robot $i$, $v_i$ is a vertex of $V_i$ that has $k$ neighbors in $G_i$ at distance at most $\varepsilon$. Then the vertex $(v_1,...,v_m)\in \dV$ might have as much as $k^m$ neighbors that are at distance
at most $\varepsilon\sqrt{m}$ in $\dG$.

\section{Future work}
\label{sec:future}
\noindent\textbf{Towards optimality.}
    Currently, our algorithmic framework is concerned with finding \emph{some} solution. Our immediate future goal is to modify it to provide a solution with quality guarantees, possibly by taking an approach similar to the continuous RRT* algorithm~\cite{KF11}, which is known to be asymptotically optimal. A fundamental difference between RRT* and the original formulation of RRT is in a rewiring step, where the structure of the tree is revised to improve previously examined paths. Specifically, when a new node is added to the tree, it is checked as to whether it will be more beneficial for some of the existing nodes to point to the new vertex instead of their current parent in the tree. This can be adapted, to some extent, to the discrete case, although it is not clear whether this indeed will lead to optimal paths.

\vspace{3pt}
\noindent\textbf{dRRT in other settings of motion planning.}
In this paper we combined the dRRT algorithm with implicit composite roadmaps to provide an efficient algorithm for multi-robot motion planning. One of the benefits of our framework comes from the fact that it reuses some of the already computed information to avoid performing costly operations.    
In particular, it refrains from checking collisions between robots with obstacles by forcing the individual robots to move on precalculated individual roadmaps (i.e., $G_i$). We believe that a similar approach can be used in other settings of motion planning. In particular, we are currently working on a dRRT-based approach for motion planning of a multi-linked robot. The new approach generates an implicitly-represented roadmap, which encapsulates information on configurations and paths between configuration that do not induce self-intersections of the robot, while ignoring the existence of obstacles. Then, we overlay this roadmap on the workspace,  an operation which invalidates some of the nodes and edges of the roadmap. Thus, we know only which configurations are self-collision free, but not obstacles collision-free. Then we use dRRT for pathfinding on the new roadmap, while avoiding self-collision tests and while exploring a small portion of the infinite roadmap.

\section{Acknowledgements}\label{sec:ack}

We wish to thank
Glenn Wagner for advising on the M* algorithm and
Ariel Felner for advice regarding pathfinding algorithms on graphs.
We note that the title ``Finding a Needle in an Exponential Haystack'' has been previously used in a talk by Joel Spencer in a different context.


\end{document}